\documentclass[12pt, a4paper]{article}

\usepackage[utf8]{inputenc}

\usepackage{amsmath}
\usepackage{amssymb}
\usepackage{amsthm}
\usepackage{newtxtext, newtxmath} 

\usepackage[margin=1in]{geometry}
\usepackage{hyperref}
\usepackage{xurl} 
\hypersetup{
    colorlinks=true,
    linkcolor=blue,
    filecolor=magenta,
    urlcolor=cyan,
    pdftitle={Computational Irreducibility as the Foundation of Agency},
    pdfauthor={Poria Azadi},
    breaklinks=true 
}

\theoremstyle{plain} 
\newtheorem{theorem}{Theorem}
\newtheorem{corollary}{Corollary}
\newtheorem{lemma}{Lemma}
\newtheorem{proposition}{Proposition}

\theoremstyle{definition} 
\newtheorem{definition}{Definition}

\title{Computational Irreducibility as the Foundation of Agency: A Formal Model Connecting Undecidability to Autonomous Behavior in Complex Systems}
\author{Poria Azadi \\ Tehran, Iran \\ \href{mailto:poria.azadi96@gmail.com}{poria.azadi96@gmail.com} \\ ORCID: 0000-0002-7692-6833}
\date{} 

\begin{document}

\maketitle

\begin{abstract}
This article presents a formal model demonstrating that genuine autonomy, the ability of a system to self-regulate and pursue objectives, fundamentally implies computational unpredictability from an external perspective. We establish precise mathematical connections, proving that for any truly autonomous system, questions about its future behavior are fundamentally undecidable. This formal undecidability, rather than mere complexity, grounds a principled distinction between autonomous and non-autonomous systems. Our framework integrates insights from computational theory and biology, particularly regarding emergent agency and computational irreducibility, to explain how novel information and purpose can arise within a physical universe. The findings have significant implications for artificial intelligence, biological modeling, and philosophical concepts like free will.
\end{abstract}

\noindent\textbf{Keywords:} Agency, Autonomy, Computational Constraints, Undecidability, Complex Systems

\section{Introduction}
The concepts of agency and autonomy, pertaining to a system's ability to function effectively, pursue objectives, and self-regulate, are central inquiries across biology, cognitive science, artificial intelligence, and philosophy \cite{moreno2015, hofmann2024}. Furthermore, previous formalization attempts, largely focused on logical or probabilistic frameworks, have frequently overlooked the inherent limitations imposed by computational constraints on a system's capacity to process information and forecast its environment \cite{froese2007, barandiaran2009}.

This article posits that a deeper understanding of agency can be achieved by examining the fundamental constraints of computation and logic within complex systems. Building upon insights from Gödel's incompleteness theorems \cite{godel1931}, Turing's work on decidability and computability \cite{turing1936}, and concepts from thermodynamics and information theory, we formulate a novel explanation. Our core thesis is that genuine autonomy necessarily implies unpredictability from an external perspective: for any truly autonomous system, there exist questions about its future behavior that are fundamentally undecidable. This provides a principled distinction between autonomous and non-autonomous systems without appealing to non-physical properties. We contend that agency specifically emerges in systems operating at the threshold of decidability. Here, Gödel-like constraints manifest as the system's inability to internally represent complete predictive models of its surroundings \cite{hofstadter2007}, while Turing-like undecidability signifies the inability to algorithmically ascertain certain future states without executing the entire computation. This leads to computational irreducibility, which generates an information gradient crucial for adaptive behavior \cite{chaitin1987}, implying that effective prediction necessitates direct simulation. We argue that this formal necessity is pivotal for comprehending authentic agency, requiring systems to generate novel computational solutions rather than merely retrieving pre-computed responses.

To explore this, we propose a formal model of a "minimal agent" interacting with an environment, clarifying essential concepts such as decidability, completeness, predictability, computational irreducibility, entropy, and information within this framework. We then introduce and formally prove theoretical assertions demonstrating how the undecidability of specific inquiries regarding an agent's long-term goal-achieving behavior leads to its computational irreducibility, resulting in a lack of external predictability and the continuous generation of new information. This highlights the inherent need for autonomous action (agency) for the agent's survival or goal accomplishment.

To avoid circular reasoning and clarify the logical dependencies in our argument, we present the following logical structure of definitions and theorems:
We begin with fundamental definitions (Definition 1 - Definition 4) of decidability, completeness, and computational irreducibility from established literature in theoretical computer science.
We then propose Definition 4 of autonomy based on specific computational properties that can be independently verified, without presupposing Turing-completeness or undecidability.
Theorem 1 then establishes that an autonomous agent, as defined in Definition 4, must necessarily be Turing-complete if it aims to achieve arbitrary goals.
Corollary 1 follows from Theorem 1, showing that autonomous systems must be capable of performing computationally irreducible processes.
Theorem 2 demonstrates that for autonomous agents, there exist undecidable questions about their goal-achieving behavior.
Theorem 3 establishes the formal connection between undecidability and computational irreducibility.
Finally, Theorem 2 (Autonomy Implies Undecidability), our central result, provides the formal proof that autonomy necessarily implies the existence of undecidable properties about the agent's behavior.

The overall structure of the paper is as follows: Section 1 introduces the topic and outlines the logical framework. Section 2 establishes the mathematical foundation and defines key concepts. Section 3 introduces the formal model of a minimal agent and its biological foundations. Section 4 presents our theoretical results, including proofs for autonomy implying Turing-completeness, undecidability, and computational irreducibility. Section 5 discusses emergent agency with illustrative examples. Section 6 details the theoretical implications and limitations. Finally, Section 7 outlines future research directions and concludes the paper.

\section{Foundational Concepts and Definitions}
\begin{definition}[Decidability]\label{def:decidability}
A decision problem, denoted as $P$, is decidable if and only if there exists a Turing machine, $M$, such that:
\begin{enumerate}
    \item $M$ halts on every input $w$.
    \item $M$ accepts $w$ if and only if $w$ is a positive instance of $P$.
    \item $M$ rejects $w$ if and only if $w$ is a negative instance of $P$.
\end{enumerate}
Formally, a total computable function $f:\Sigma^* \to \{0,1\}$ must exist. For all inputs $w \in \Sigma^*$ (where $\Sigma^*$ is the set of all possible strings over a finite alphabet), this function $f(w)=1$ if $w \in P$ and $f(w)=0$ if $w \notin P$.

In the context of agent behavior, a property $P$ of an agent-environment system is decidable if an algorithm exists that, given a description of the agent, environment, and initial conditions, can determine in finite time whether property $P$ will ever hold during the system's evolution \cite{sipser2012}. For environments exhibiting sufficient complexity (e.g., those equivalent to Class 4 cellular automata as described by Wolfram \cite{wolfram2002}), determining certain future properties can be provably undecidable. This is often shown via reduction to the Halting Problem \cite{wolfram1985}, which asks whether an arbitrary program will finish running or continue to run forever.
\end{definition}

\begin{definition}[Completeness]\label{def:completeness}
A formal system $S$ is complete if and only if for every statement $\phi$ formulated in the language of $S$, either the statement itself or its negation is provable within the system $S$. Formally, for all statements $\phi$ in the language $L(S)$ of the system, either $S \vdash \phi$ (meaning $\phi$ is provable in $S$) or $S \vdash \neg\phi$ (meaning $\neg\phi$ is provable in $S$). A complete system, therefore, can, in principle, decide the truth or falsity of any statement expressible in its language.

For an agent system, let $S_A$ be the agent's internal representational system (e.g., its knowledge base and reasoning rules), and let $S_E$ represent the set of true statements about the environment. Completeness in this context means that for every truth about $S_E$ that is necessary for the agent's goal achievement, either $\phi$ or $\neg\phi$ is derivable within $S_A$. However, Gödel's first incompleteness theorem \cite{godel1931} states that any consistent formal system $S$ capable of representing elementary arithmetic contains statements that are true but unprovable within $S$. Therefore, for any sufficiently complex agent possessing a finite representational capacity, there will exist true environmental statements relevant to its goals that cannot be derived (proven true or false) using only its internal model \cite{hofstadter2007}.
\end{definition}

\begin{definition}[Computational Irreducibility]\label{def:computational_irreducibility}
A process $P: I \to O$ (mapping inputs $I$ to outputs $O$) is computationally irreducible if and only if, for almost all inputs $i \in I$, there exists no significantly more efficient algorithm to compute $P(i)$ than directly simulating the process $P$ itself. In other words, one cannot find a "shortcut" to predict the outcome without essentially running the process.

Formally, a process $P$ is computationally irreducible if, for almost all inputs $i \in I$, the conditional Kolmogorov complexity $K(P(i)|i)$ approaches the computational resources required to execute P on i. The conditional Kolmogorov complexity $K(x|y)$ is the length of the shortest computer program that produces output $x$ given input $y$. Specifically, for a universal Turing machine $U$, a process $P$ is computationally irreducible if there exists a constant $c$ such that for almost all inputs $i \in I$:
\begin{equation}
K_U(P(i)|i) \geq \text{TimeComplexity}_P(i) - c
\end{equation}
Here, $K_U(P(i)|i)$ is the conditional Kolmogorov complexity of the output $P(i)$ given the input $i$, and $\text{TimeComplexity}_P(i)$ is the time complexity (number of computational steps) of executing process $P$ on input $i$. This definition captures the notion that no significantly shorter description or faster method exists to determine the outcome of $P$ without performing the computation itself \cite{wolfram2002, wolfram1985}. This concept is central to understanding the limits of prediction in complex systems and potentially links to thermodynamic principles \cite{zenil2021}. Note that this definition focuses on the time or effort required for prediction.

\textbf{Link to Predictability:} Undecidability directly implies limitations on predictability. If determining whether a future property will hold is an undecidable problem, then no general algorithm exists that can always predict it. Computational irreducibility implies practical unpredictability because any prediction requires a simulation of complexity comparable to the process itself. We can formalize predictability using resource-bounded computation. Let $\text{Pred}_r(s,t)$ be a prediction function that uses at most $r$ computational resources (e.g., time steps) to predict the state of a system $s$ at a future time $t$. The prediction efficiency $\eta(r,t,s) = D(\text{Actual}(s,t), \text{Pred}_r(s,t))/r$ measures the accuracy of the prediction (where $D$ is a distance metric comparing the predicted state to the actual state) per unit of resource used. Computational irreducibility implies that for a fixed amount of computational resources $r$, the prediction efficiency $\eta(r,t,s)$ tends to zero as the prediction horizon $t$ increases, for almost all initial states $s$. Formally:
\[
\lim_{t \to \infty} \eta(r,t,s) = 0 \quad \text{for fixed } r \text{ and almost all } s.
\]
\end{definition}

\begin{definition}[Autonomy -- Computationally Grounded]\label{def:autonomy}
An agent $A$ interacting with an environment $E$ exhibits computationally grounded autonomy if and only if:
\begin{enumerate}
    \item \textbf{Internal State Independence:} Agent $A$ possesses an internal state space $S_A$ and transition function $T_A$ that can maintain and update internal states not solely determined by immediate environmental inputs. This enables the agent to operate based on its history and internal logic rather than just reactive responses.
    \item \textbf{Generative Decision-Making:} The agent's transition function $T_A$ includes a component $g(s_A)$ derived from the agent's internal logic, allowing it to generate actions not fully predetermined by external stimuli. This function enables the agent to act as a source of novel behaviors.
    \item \textbf{Effective Environmental Coupling:} The agent has the ability to read from and write to its environment, thus establishing a bidirectional information flow between its internal states and the external world.
\end{enumerate}
This definition characterizes autonomy through the agent's operational properties without presupposing any specific computational power.
\end{definition}

\begin{definition}[Operational Closure Formal]\label{def:operational_closure}
A system $S$ is operationally closed if and only if:
\begin{enumerate}
    \item \textbf{Structural Self-Reference:} System $S$ possesses a set of processes $P$ operating in a recursive network, such that for every process $p \in P$, the inputs to $p$ are produced by the outputs of other processes within $P$.
    \item \textbf{Dynamic Independence:} A transition function $T_S: S \times E \to S$ exists, where S is the set of system states and E is the set of environmental states, such that: $\forall s \in S, \forall e, e' \in E: \exists \delta > 0$ such that if $\lVert e - e' \rVert < \delta$, then $\lVert T_S(s,e) - T_S(s,e') \rVert < \epsilon$. This allows the system to maintain relative stability against small environmental perturbations.
    \item \textbf{Algorithmic Limitation:} There exists no algorithm $\mathcal{A}$ with computational resources $r$ such that for any environmental state $e \in E$ and any time $t > t_0$: $\text{Pred}_{\mathcal{A},r}(S,e,t) = \text{Actual}_S(e,t)$, where $\text{Pred}_{\mathcal{A},r}$ is algorithm $\mathcal{A}$'s prediction of the system's state and $\text{Actual}_S$ is the system's actual state. This implies that no algorithm can accurately and more efficiently predict the system's behavior over time than direct simulation.
    \item \textbf{State-Space Distinction:} The set $S$ comprises states that are topologically distinct from $E$, meaning there exists a boundary $M: S \times E \to \{0,1\}$ indicating which variables belong to the system and which to the environment.
    \item \textbf{Core Stability:} A subset $C \subset S$ of core states exists such that for any $c \in C$, the system tends to remain in states near C, and any deviation from C is met with compensatory mechanisms.
\end{enumerate}
\end{definition}

\begin{definition}[Emergent Agency Formal]\label{def:emergent_agency}
Capacity to reliably achieve/maintain goal $G$ in an irreducible environment via processes not fully predictable externally, emerging from real-time interaction \cite{klyubin2005, wang2024}.
\end{definition}

\begin{definition}[Computational Sourcehood]\label{def:computational_sourcehood}
An agent $A$ exhibits computational sourcehood if and only if:
\begin{enumerate}
    \item \textbf{Prediction Requirement:} Any algorithm $P$ that accurately predicts $A$'s behavior must implement a computational structure that preserves essential features of $A$'s internal computational organization.
    \item \textbf{Formal Structure Preservation:} There exists a mapping function $f$ from $P$'s computational states to $A$'s internal states that preserves the causal structure of $A$'s computation.
    \item \textbf{Work Cycle Implementation:} The agent implements work cycles in Kauffman and Clayton's sense \cite{kauffman2006}, harnessing energy from the environment to maintain its organization.
    \item \textbf{Organic Code Implementation:} The agent employs organic codes in Barbieri's sense \cite{barbieri2003}, establishing arbitrary correspondences between independent domains.
    \item \textbf{Irreducible Representation:} There exists no algorithm $Q$ that can predict $A$'s behavior with equal accuracy while implementing a computational structure that does not preserve these essential features.
\end{enumerate}
This definition formalizes Kauffman and Clayton's philosophical claim that "no theory in physics or chemistry can predict the emergence of meanings and purposes" \cite{kauffman2006} by grounding it in the computational concepts of undecidability and irreducibility established in this paper. Barbieri's code biology provides a complementary framework for understanding emergent agency through the lens of organic codes. According to Barbieri (2008) \cite{barbieri2008}, life fundamentally depends on semiosis, defined as the production of signs and codes. Unlike interpretive semiosis that requires consciousness, organic semiosis operates through coding rules that establish correspondence between two independent worlds. This aligns with our computational model by demonstrating how meaning and functional information can emerge from rule-based systems without requiring interpretive agents \cite{barbieri2008}.
\end{definition}

\section{Minimal Agent Model and Biological Foundations}
We define a minimal agent $A$ interacting with an environment $E$.

\begin{definition}[Environment]\label{def:environment}
The environment $E$ is a dynamic system. It has a state space, denoted $S_E$, and a transition function $T_E: S_E \times O_A \to S_E$. Here, $O_A$ is the set of actions available to the agent. The environment's next state $s_E(t+1)$ is determined by its current state $s_E(t)$ and the agent's action $o_A(t)$: $s_E(t+1) = T_E(s_E(t), o_A(t))$. An environment $E$ can be characterized by its computational complexity class $C(E)$ (if $T_E$ belongs to class $C$). Importantly, $E$ can be Turing-complete, meaning it can perform universal computation. The complexity of the environment can be measured by its conditional entropy $\lambda_E = H(S_E(t+1) | S_E(t), O_A(t))$, which represents the average amount of new information or uncertainty generated by the environment at each transition, given the previous state and the agent's action. We propose that for agency to emerge, $\lambda_E$ must be sufficiently high, indicating a rich and somewhat unpredictable environment \cite{beer2014}.
\end{definition}

\begin{definition}[Minimal Agent]\label{def:minimal_agent}
The minimal agent $A$ is a computational system described by the following components:
\begin{itemize}
    \item A state space $S_A$: This represents the set of all possible internal configurations of the agent.
    \item An input function $I_A: S_E \to I'_A$: This function maps states of the environment $S_E$ to the agent's internal inputs $I'_A$ (where $I'_A$ is the set of possible inputs the agent can perceive).
    \item A transition function $T_A: S_A \times I'_A \to S_A$: This function determines the agent's next internal state based on its current internal state $s_A(t)$ and its current input $i_A(t)$: $s_A(t+1) = T_A(s_A(t), i_A(t))$.
    \item An output function $O_A: S_A \to O'_A$: This function maps the agent's internal states to actions $O'_A$ (where $O'_A$ is the set of possible actions the agent can perform): $o_A(t) = O_A(s_A(t))$.
    \item A goal state or condition $G \subseteq S_A \times S_E$: This defines the agent's objective(s) as a subset of combined agent-environment states.
\end{itemize}
The degree of an agent's autonomy can potentially be measured by comparing the mutual information $I(S_A(t); S_E(t+1))$ (flow from agent to environment) versus $I(S_E(t); S_A(t+1))$ (flow from environment to agent) \cite{bertschinger2008}.
\end{definition}

\begin{definition}[Biologically Grounded Minimal Agent]\label{def:biologically_grounded_minimal_agent}
Building upon Definition \ref{def:minimal_agent}, a biologically grounded minimal agent additionally satisfies:
\begin{enumerate}
    \item[a.] \textbf{Organic Code Implementation:} The agent's transition function $T_A$ implements organic codes in Barbieri's sense \cite{barbieri2015}, establishing correspondence rules between environmental states and internal representations, similar to how ribosomes implement the genetic code by establishing correspondence between nucleotide sequences and amino acids.
    \item[b.] \textbf{Work-cycle Implementation:} The agent's transition function $T_A$ implements work cycles that enable it to extract and utilize energy from its environment, corresponding to Kauffman and Clayton's concept of work cycles \cite{kauffman2006}.
    \item[c.] \textbf{Self-catalytic Processes:} The generative decision-making component $g(s_A)$ exhibits self-catalytic properties, where the agent's outputs contribute to maintaining the conditions necessary for its continued operation.
    \item[d.] \textbf{Computational Irreducibility:} As a consequence of these biological properties, the agent necessarily exhibits the computational properties established in Theorem \ref{thm:turing_completeness} (Turing-completeness), Theorem \ref{thm:undecidability} (undecidability), and Theorem \ref{thm:undecidability_irreducibility} (computational irreducibility).
\end{enumerate}
This definition bridges the gap between abstract computational properties and their biological implementations, showing how the mathematical constraints we have identified may be instantiated in living systems.
\end{definition}

\section{Theoretical Results}
\subsection{Computational Power of Autonomous Agents}

\begin{theorem}[Autonomy Necessitates Turing-Completeness]\label{thm:turing_completeness}
Any system $A$ exhibiting autonomy (according to Definition \ref{def:autonomy}) that aims to achieve arbitrary goal states must be Turing-complete.
\end{theorem}
\begin{proof}
\textbf{Formal Statement:} Let $A$ be an agent.
\[ (\text{A exhibits Autonomy}) \land (\text{A can achieve arbitrary goals}) \implies (\text{A is Turing-complete}) \]
\vspace{\topsep} 
Let us first establish notation:
Let $C(A)$ denote the computational class to which agent A belongs.
Let $C_{TM}$ denote the class of Turing-complete systems.
Let $C_{FA}, C_{PDA}, C_{LBA}$ denote finite automata, pushdown automata, and linear bounded automata respectively.

\begin{lemma}[Formal]\label{lemma:decidable_sub_turing}
If $C(A) \subset C_{TM}$ (i.e., A belongs to a computational class less powerful than Turing-complete), then all behavioral properties P expressible within $C(A)$ are decidable by external algorithms.
\end{lemma}
\begin{proof}[Proof of Lemma 1]
For any computational model $M \in \{C_{FA}, C_{PDA}, C_{LBA}\}$, we know from computability theory that:
For $M=C_{FA}$:
State reachability problem: Given states $q_1, q_2 \in Q$, determine if $\exists w \in \Sigma^* : \delta^*(q_1, w) = q_2$ is decidable in $O(|Q|^2)$ time.
Language emptiness: Given $L(M)$, determine if $L(M) = \emptyset$ is decidable in $O(|Q|)$ time.
For $M=C_{PDA}$:
State reachability is decidable (via conversion to CFG and CYK algorithm).
Language emptiness is decidable in polynomial time.
For $M=C_{LBA}$:
Termination is decidable (by enumerating all possible configurations, which are finitely many).
Let $P$ be any behavioral property expressible within $C(A)$. Since $C(A) \subset C_{TM}$, $A$ falls into one of the classes above or their equivalents. Therefore, $P$ is decidable by the corresponding decision algorithm.
\end{proof}

\textbf{Main Proof:} We now proceed by constructive proof showing that an autonomous agent $A$ satisfying Definition \ref{def:autonomy} can simulate a Universal Turing Machine (UTM).
Let $U = (Q_U, \Sigma_U, \Gamma_U, \delta_U, q_{0U}, q_{acceptU}, q_{rejectU})$ be a UTM.
Given an agent $A$ with properties from Definition \ref{def:autonomy}:
\begin{enumerate}
    \item \textbf{Internal State Independence:} A has internal state space $S_A$ and transition function $T_A$.
    \item \textbf{Generative Decision-Making:} A's transition function includes component $g(s_A)$.
    \item \textbf{Effective Environmental Coupling:} A can read from and write to its environment $E$.
\end{enumerate}
We construct a goal-directed task $G$ that requires $A$ to simulate $U$ as follows:
\begin{itemize}
    \item \textbf{State Mapping:} Define an injective mapping $\phi: Q_U \to S_A$. This is possible due to condition 1.
    \item \textbf{Transition Function Simulation:} For each UTM transition $\delta_U(q,\sigma) = (q', \sigma', d)$, define $g(\phi(q), \text{read}(\sigma)) = (\phi(q'), \text{write}(\sigma'), \text{move}(d))$. This is possible due to condition 2.
    \item \textbf{Tape Simulation:} The agent uses the environment $E$ as the UTM's tape via functions read, write, and move. This is possible due to condition 3.
    \item \textbf{Goal State Definition:} Define $G = \{(s_A, s_E) | s_A = \phi(q_{acceptU})\}$.
\end{itemize}
By construction, agent $A$ can simulate any computation of UTM $U$. Therefore, $A$ is Turing-complete.
Now, by contradiction, assume $A$ is not Turing-complete but still exhibits autonomy per Definition \ref{def:autonomy}. Then by Lemma \ref{lemma:decidable_sub_turing}, all behavioral properties of $A$ would be decidable. But we showed $A$ can simulate a UTM, which has undecidable properties (e.g., the Halting Problem). This contradiction implies that any agent satisfying Definition \ref{def:autonomy} must be Turing-complete if it aims to achieve arbitrary goal states.
\end{proof}

\begin{corollary}[Computational Irreducibility in Autonomous Systems]\label{cor:irreducibility}
Every autonomous system (as defined in Definition \ref{def:autonomy}) is capable of performing computations that exhibit computational irreducibility.
\end{corollary}
\begin{proof}
Let $A$ be an autonomous system according to Definition \ref{def:autonomy}.
\textbf{Step 1: Establish the existence of computationally irreducible processes.}
Let $C_{IR} \subset C_{TM}$ be the set of computationally irreducible processes. By the work of Cook \cite{cook2004} and Wolfram \cite{wolfram2002}, this set is non-empty. Specifically, Rule110 $\in C_{IR}$.
\textbf{Step 2: Apply Theorem \ref{thm:turing_completeness}.}
By Theorem \ref{thm:turing_completeness}, since $A$ is autonomous, $C(A) = C_{TM}$.
\textbf{Step 3: Prove simulation capability.}
Since $A$ is Turing-complete, it can simulate any computation in $C_{TM}$, including any $C \in C_{IR}$.
\textbf{Step 4: Establish computational irreducibility inheritance.}
Let $A_{Rule110}$ be the configuration of $A$ that simulates Rule 110. Assume, for contradiction, that $A_{Rule110}$ is not computationally irreducible. This would imply the existence of a shortcut to compute Rule 110, contradicting its established irreducibility.
Therefore, $A_{Rule110}$ must be computationally irreducible.
Since $A$ can be configured to simulate an irreducible process, every autonomous system is capable of performing computationally irreducible computations.
\end{proof}

\subsection{Undecidability and Irreducibility}
This subsection explores the consequences of these computational limits within our formal model of an agent, linking undecidability to unpredictability and information generation.

\begin{definition}[External Prediction Problem]\label{def:external_prediction_problem}
Given agent $A$, environment $E$, initial state $s_0 = (s_A(0), s_E(0))$, and property $P: S_A \times S_E \to \{0,1\}$, the External Prediction Problem $EP(A,E,P,s_0)$ asks: "Will the system $(A,E)$, starting from state $s_0$, ever reach a state where $P$ holds?"
Formally, $EP(A,E,P,s_0) = 1$ if and only if $\exists t \in \mathbb{N} : P(s_A(t), s_E(t)) = 1$.
\end{definition}

\begin{definition}[Decidability of Agent Behavior]\label{def:decidability_agent_behavior}
The behavior of $(A,E)$ is externally decidable for property $P$ if there exists a total Turing machine $D_P$ that solves $EP(A,E,P,s_0)$ for any $s_0$. If no such $D_P$ exists, then $EP(A,E,P,s_0)$ is undecidable.
\end{definition}

\begin{theorem}[Autonomy Implies Undecidability]\label{thm:undecidability}
If an agent $A$ exhibits autonomy according to Definition \ref{def:autonomy}, then there exists at least one non-trivial property $P$ such that the External Prediction Problem $EP(A,E,P,s_0)$ is undecidable.
\end{theorem}
\begin{proof}
\textbf{Formal Statement:}
\[ (\text{A exhibits Autonomy}) \land (\text{A can achieve arbitrary goals}) \implies (\text{A is Turing-complete}) \]
\vspace{\topsep}
\begin{lemma}\label{lemma:simulate_utm}
An autonomous agent $A$ (according to Definition \ref{def:autonomy}), interacting with a suitable environment $E$, can simulate a Universal Turing Machine (UTM).
\end{lemma}
\begin{proof}[Proof of Lemma 2]
This follows directly from the constructive proof of Theorem \ref{thm:turing_completeness}, where we showed that an agent satisfying the conditions of Definition \ref{def:autonomy} can be configured to simulate a UTM by mapping UTM states to agent states, UTM transitions to the agent's generative logic, and the UTM tape to the environment.
\end{proof}

\textbf{Construction of an Undecidable Property:}
Let $U_{std}$ be a fixed standard Universal Turing Machine. For any arbitrary Turing Machine $M_{arb}$ and input $w_{arb}$, let $d_{M_{arb},w_{arb}} = \langle M_{arb}, w_{arb} \rangle$ be the encoding as input for $U_{std}$ to simulate $(M_{arb}, w_{arb})$.
Set the initial state $s_0$ of $(A,E)$ to represent $U_{std}$ starting on input $d_{M_{arb},w_{arb}}$.
Define property $P_{halt}(s_A, s_E) = 1$ if and only if $s_A \in \{\phi(q_{acceptU_{std}}), \phi(q_{rejectU_{std}})\}$.

\textbf{Analysis of $EP(A,E,P_{halt},s_0)$:}
The problem asks if the system will ever reach a halting state. By our construction, this is equivalent to asking if $U_{std}$ halts on input $d_{M_{arb},w_{arb}}$, which in turn is equivalent to the standard Halting Problem for $(M_{arb}, w_{arb})$.

\textbf{Proof by Contradiction:}
Assume $EP(A,E,P_{halt},s_0)$ is decidable by some algorithm $D_{predict}$. Then we could construct an algorithm $D_{halt}$ that solves the Halting Problem by encoding $(M_{arb}, w_{arb})$, constructing the corresponding agent system $(A,E,s_0)$, and running $D_{predict}$. But the Halting Problem is known to be undecidable \cite{turing1936}. This contradiction implies our assumption is false, and $EP(A,E,P_{halt},s_0)$ must be undecidable. $P_{halt}$ is non-trivial as some inputs halt and others do not. This result can be generalized via Rice's Theorem \cite{rice1953} to any non-trivial semantic property of the agent's behavior.
\end{proof}

\begin{theorem}[Connection Between Undecidability and Computational Irreducibility]\label{thm:undecidability_irreducibility}
For a property $P$ related to the long-term behavior of the system $(A,E)$ that is undecidable, there exists a set $S_P$ of initial states for which the system's evolution $s(0) \to s(t)$ is computationally irreducible for $s(0) \in S_P$.
\end{theorem}
\begin{proof}
\textbf{Formal Statement:}
\[ (\text{EP}(A,E,P) \text{ is undecidable}) \implies (\exists s(0) \in S_P : \text{evolution is computationally irreducible}) \]
\vspace{\topsep}
Let $P$ be an undecidable property of the system $(A,E)$, and let $S_P = \{s(0) | EP(A,E,P,s(0)) \text{ is undecidable}\}$. $S_P$ is non-empty.
Assume for contradiction that for all $s(0) \in S_P$, the system's evolution is computationally reducible. This implies the existence of a prediction algorithm $R$ that computes $s(t)$ from $s(0)$ significantly faster than direct simulation.
We could use $R$ to construct a semi-decision procedure $D_P$ for property $P$: for $t=1, 2, ...$, compute $s(t) = R(s(0), t)$ and check if $P(s(t))$ holds. If it does, halt and return "Yes".
This procedure, $D_P$, would correctly identify all positive instances of $P$. If we could also construct a semi-decision procedure for $\neg P$, we could run them in parallel to create a full decision procedure for $P$. The existence of such procedures for certain undecidable problems (like the Halting Problem) would contradict the undecidability of P.
Therefore, our assumption must be false. There must exist at least one initial state $s(0) \in S_P$ for which the system's evolution is computationally irreducible.
\end{proof}

\begin{proposition}[Lack of External Predictability]\label{prop:no_external_predictability}
For a computationally irreducible system (as per Definition \ref{def:computational_irreducibility}) and for almost all initial states $s(0)$, no general algorithm can predict the system's future state $s(t)$ with significantly less computational effort than performing the direct simulation for $t$ steps.
\end{proposition}
\begin{proof}
This follows directly from Definition \ref{def:computational_irreducibility}. Computational irreducibility states that for almost all inputs (initial states $s(0)$), any algorithm predicting the outcome after $t$ steps must have a time complexity that is at least proportional to $t$. Therefore, no significant shortcut exists compared to the direct simulation.
\end{proof}

\begin{proposition}[Information Generation in Irreducible Systems]\label{prop:information_generation}
Consider a computationally irreducible system where future states are uniquely determined by initial conditions. Such a system generates novel information over time relative to any resource-bounded observer, in a well-defined formal sense.
\end{proposition}
\begin{proof}
While the Shannon entropy $H(s(t)|s(0)) = 0$ for a deterministic system, information generation can be understood through Kolmogorov complexity. For a computationally irreducible system, the conditional Kolmogorov complexity of the trajectory $X_t = (s_0, s_1, ..., s_t)$, denoted $K(X_t|s_0)$, grows linearly with time: $K(X_t|s_0) = \Omega(t)$ \cite{li2008}. This means each time step adds a constant amount of irreducible, incompressible information to the trajectory's description. This new information cannot be accessed by any resource-bounded observer without performing the computation itself, thus appearing as "novel" information from the observer's perspective.
\end{proof}

\begin{proposition}[Absence of General Analytical Solution]\label{prop:no_analytical_solution}
For a computationally irreducible agent-environment system, no general closed-form analytical solution $f(s(0), t) = s(t)$ exists that is computable significantly faster than $T_{Sim}(s(0), t)$.
\end{proposition}
\begin{proof}
Assume, for contradiction, that such a solution $f$ exists and is computable in time $o(t)$ (e.g., $O(\text{poly}(\log t))$). This would mean we have an algorithm that computes the system's state $s(t)$ with a time complexity significantly less than that of direct simulation ($t$ steps). This directly contradicts the definition of computational irreducibility, which states that for almost all initial states, no such "shortcut" algorithm exists. Therefore, no general, fast, closed-form solution can exist.
\end{proof}

\begin{theorem}[Organic Codes and Computational Irreducibility]\label{thm:organic_codes}
Biological systems employing multiple interacting organic codes (in Barbieri's sense) necessarily exhibit computational irreducibility.
\end{theorem}
\begin{proof}
\textbf{Formal Statement:}
\[ (\text{System S has multiple interacting organic codes}) \implies (\text{S is computationally irreducible}) \]
\vspace{\topsep}
Let a biological system have $n$ interacting organic codes $C_1, ..., C_n$, each implementing a mapping $f_i: X_i \to Y_i$.
\textbf{Step 1:} Each mapping $f_i$ is arbitrary and not derivable from physical laws alone \cite{barbieri2003}.
\textbf{Step 2:} The interaction of these codes forms a complex computational system. The composition of these mappings $F = f_n \circ \dots \circ f_1$ can create context-dependent rules capable of universal computation \cite{soloveichik2008}.
\textbf{Step 3:} Such biological systems satisfy our Definition \ref{def:autonomy} of autonomy. By Theorem \ref{thm:turing_completeness}, they must be Turing-complete.
\textbf{Step 4:} By Theorem \ref{thm:undecidability}, Turing-complete autonomous systems have undecidable properties. By Theorem \ref{thm:undecidability_irreducibility}, undecidability implies computational irreducibility.
Therefore, biological systems with multiple interacting organic codes are necessarily computationally irreducible.
\end{proof}

\section{Emergent Agency and Illustrations}

\begin{proposition}[Computational Irreducibility as a Basis for Agency]\label{prop:basis_for_agency}
For an agent A interacting with an environment E exhibiting computational irreducibility for certain processes, emergent agency arises when:
(i) The agent executes its internal rules in real-time in response to environmental stimuli
(ii) These rules are well-adapted to achieve or maintain goal condition G
(iii) The path to achieving G cannot be determined by analytical shortcuts
\end{proposition}
\begin{proof}
\textbf{Step 1:} The environment $E$ is computationally irreducible by premise. From Proposition \ref{prop:no_external_predictability}, its future states are unpredictable without simulation.
\textbf{Step 2:} The agent-environment system $A \times E$ evolves according to the coupled dynamics $s_A(t+1) = T_A(s_A(t), I_A(s_E(t)))$ and $s_E(t+1) = T_E(s_E(t), O_A(s_A(t)))$.
\textbf{Step 3:} As established by Theorems \ref{thm:turing_completeness}, \ref{thm:undecidability}, and \ref{thm:undecidability_irreducibility}, a successful autonomous agent operating in such an environment must be Turing-complete, will exhibit undecidable properties, and its interaction with the environment will be computationally irreducible. To achieve a goal $G$, the agent must find a path to a state in $G$. Due to irreducibility, no shortcut exists to find this path.
\textbf{Step 4:} The only guaranteed way to determine the path to goal $G$ is through the actual execution of the agent-environment interaction in real time. This satisfies Definition \ref{def:emergent_agency} of Emergent Agency because the agent succeeds by "being the computation" in real time, a process which cannot be pre-determined.
\end{proof}

\subsection{Illustrative Examples}
To demonstrate the practical implications of our formal framework, we examine examples where computational irreducibility underpins adaptive behaviors and emergent agency:

\textbf{Gene Regulatory Networks \& Morphogenesis:} Biological development showcases how local genetic and cellular rules give rise to complex, functionally organized structures—the ultimate goal being organism viability—through computationally irreducible processes \cite{kauffman1993}. The deterministic yet irreducible rules governing gene expression and cell signaling mean that the emergence of complex patterns and robust outcomes cannot be predicted without simulating the entire developmental process \cite{hanson2009}. This fulfills our Definition \ref{def:emergent_agency} (Emergent Agency).

\textbf{Neural Information Processing:} The dynamics of the brain provide a prime example of computational irreducibility supporting adaptive behavior \cite{tononi2003}. Cognitive capabilities emerge from the distributed processing within neural networks, operating without a central controller. The brain's complex evolution defies detailed prediction without near 1:1 simulation, indicating it leverages its own irreducible computation rather than attempting to predict all contingencies in advance. This aligns with Definition \ref{def:computational_sourcehood} (Computational Sourcehood).

\textbf{Autocatalytic Sets and Computational Irreducibility:} Kauffman's work on autocatalytic sets offers a concrete biological instantiation of computational irreducibility. The question, "Will molecule M ever be produced by autocatalytic set S?" can be shown to be undecidable by reducing it to the Halting Problem \cite{soloveichik2008}. The conditional Kolmogorov complexity, $K(S(t)|s(0))$, for an autocatalytic set $S$, grows linearly with time $t$, demonstrating that each step contributes irreducible informational content (Proposition \ref{prop:information_generation}).

\textbf{Development as Computational Sourcehood:} The process of biological development (embryogenesis) distinctly exemplifies computational sourcehood (Definition \ref{def:computational_sourcehood}). The organism itself becomes the irreducible source of its own organizational complexity through the actual execution of these developmental algorithms in real-time.

\section{Implications and Limitations}
\subsection{Theoretical Implications}

\textbf{Philosophy (Free Will \& Sourcehood):} It is important to clarify that our discussion of compatibilism applies strictly within formal deterministic frameworks, not to the actual physical world. Following the work of Nicolas Gisin and Flavio Del Santo \cite{gisin2021}, we must acknowledge that the real world is fundamentally indeterministic in nature. Any discussion of compatibilism should be understood within the context that we live in an indeterministic world. As formally proven in Theorem \ref{thm:undecidability}, autonomy necessarily implies undecidability, which provides a principled basis for understanding sourcehood.

\begin{definition}[Enhanced Computational Sourcehood]\label{def:enhanced_computational_sourcehood}
It should be emphasized that artificial intelligence systems, regardless of their complexity, cannot possess genuine agency or autonomy. As Johannes Jaeger \cite{jaeger2021} conclusively demonstrates, AI remains constrained to algorithmic frameworks that are fundamentally incapable of generating true agency. Any computational sourcehood attributed to artificial systems is merely a structural simulation of agency, not its actual embodiment.
Expanding Definition \ref{def:computational_sourcehood}, an agent A exhibits enhanced computational sourcehood if, in addition to the formal requirements previously specified:
\begin{enumerate}
    \item \textbf{Context-Sensitive Selection:} The agent can distinguish between environmental contexts relevant to its goals and adapt its computational processes accordingly.
    \item \textbf{Non-prestatable Action Space:} The agent generates actions that cannot be prestated or enumerated in advance of their actual occurrence.
    \item \textbf{Propagating Constraint Construction:} The agent's computational structure enables it to construct and propagate constraints that channel energy flows in service of its continued operation.
\end{enumerate}
\end{definition}

\subsection{Limitations and Critical Assessment}
\textbf{Empirical Verification:} A significant challenge is empirical verification. Proving that a physical system implements computationally irreducible processes faces theoretical and practical obstacles \cite{israeli2012}.

\textbf{Distinguishing Types of Irreducibility:} Our framework would benefit from more refined methods to distinguish between different types of computational irreducibility, such as Bennett's logical depth \cite{bennett1988} and effective complexity \cite{gellmann2003}.

\textbf{Finite Resources:} Our framework often assumes unbounded computational resources, while real systems operate under strict constraints. Future work should address how resource bounds affect the manifestation of autonomy.

\textbf{Micro-Macro Link:} Our framework requires further development to specify the relationship between micro-level computational processes and macro-level autonomous behaviors \cite{hoel2017}.

\textbf{AI Safety Implications:} The unpredictability inherent in our framework poses challenges for reliable AI systems. As established by Definition \ref{def:enhanced_computational_sourcehood}, AI systems do not possess genuine agency. Therefore, safety must focus on external controls and rigorous verification of their algorithmic behavior.

\section{Conclusion}
This paper has established a formal framework linking fundamental computational limits to the emergence of autonomy and agency. Our central thesis, formally proven, is that genuine autonomy necessarily implies undecidability (Theorem \ref{thm:undecidability}) from an external perspective. We further demonstrated that this undecidability leads directly to computational irreducibility (Theorem \ref{thm:undecidability_irreducibility}), implying a lack of external predictability (Proposition \ref{prop:no_external_predictability}) and the continuous generation of novel information (Proposition \ref{prop:information_generation}) by autonomous systems.

Our key contributions include:
\begin{itemize}
    \item A non-circular and computationally grounded definition of autonomy (Definition \ref{def:autonomy}).
    \item A formal proof connecting autonomy to Turing-completeness (Theorem \ref{thm:turing_completeness}).
    \item Establishing the logical bridge between undecidability and computational irreducibility (Theorem \ref{thm:undecidability_irreducibility}).
    \item A formal characterization of computational sourcehood (Definition \ref{def:computational_sourcehood}) that grounds agency in irreducible computation.
\end{itemize}
These results suggest that computational irreducibility and computational sourcehood serve as robust foundations for understanding agency. We propose that unpredictability, rather than being a limitation to our understanding, is in fact constitutive of what it means to be a truly autonomous agent.

\subsection{Future Research Directions:}
Future research should explore the relationship between computational irreducibility and organic codes as described by Barbieri. The fact that multiple organic codes operate simultaneously in biological systems suggests that biological agency may emerge from the interaction of multiple irreducible coding systems \cite{barbieri2016}. Further research could also aim to develop empirical proxies for computational irreducibility \cite{israeli2012}, formalize degrees of autonomy and agency \cite{klyubin2005}, refine sourcehood models \cite{krumm2023}, investigate potential quantum effects, and establish verification methods for irreducible AI \cite{johansen2024}.



\begin{thebibliography}{99}
\bibitem{moreno2015} Moreno A, Mossio M (2015) Biological Autonomy: A Philosophical and Theoretical Enquiry. Springer, Dordrecht. \url{https://doi.org/10.1007/978-94-017-9837-2}
\bibitem{hofmann2024} Hofmann B, Kolling N, Bongard J, Mahadevan S (2024) Understanding biology in the age of artificial intelligence: Challenges of irreducibility and emergence. Nat Methods 21(4):440-448. \url{https://doi.org/10.1038/s41592-024-02121-z}
\bibitem{biehl2022} Biehl M, Polani D, Ikegami T (2022) Specific and Complete: Types of Agency and Interaction in Complex Systems. Entropy 24(5):697. \url{https://doi.org/10.3390/e24050697}
\bibitem{moreno2005} Moreno A, Etxeberria A (2005) Agency in natural and artificial systems. Artif Life 11(1-2):161-175. \url{https://doi.org/10.1162/1064546053278996}
\bibitem{rosen1991} Rosen R (1991) Life Itself: A Comprehensive Inquiry Into the Nature, Origin, and Fabrication of Life. Columbia University Press, New York.
\bibitem{froese2007} Froese T, Virgo N, Izquierdo E (2007) Autonomy: A review and a reappraisal. In: Advances in Artificial Life. Springer, pp 455-464. \url{https://doi.org/10.1007/978-3-540-74913-4_46}
\bibitem{barandiaran2009} Barandiaran XE, Di Paolo E, Rohde M (2009) Defining agency: Individuality, normativity, asymmetry, and spatiotemporality in action. Adapt Behav 17(5):367-386. \url{https://doi.org/10.1177/1059712309343819}
\bibitem{godel1931} Gödel K (1931) Über formal unentscheidbare Sätze der Principia Mathematica und verwandter Systeme I. Monatshefte für Mathematik und Physik 38(1):173-198. \url{https://doi.org/10.1007/BF01700690}
\bibitem{turing1936} Turing AM (1936) On computable numbers, with an application to the Entscheidungsproblem. Proc Lond Math Soc s2-42(1):230-265. \url{https://doi.org/10.1112/plms/s2-42.1.230}
\bibitem{hofstadter2007} Hofstadter DR (2007) I Am a Strange Loop. Basic Books, New York.
\bibitem{chaitin1987} Chaitin GJ (1987) Algorithmic Information Theory. Cambridge University Press. \url{https://doi.org/10.1017/CBO9780511608858}
\bibitem{krumm2023} Krumm M, Müller T (2023) Free Agency and Determinism: Is There a Sensible Definition of Computational Sourcehood? Entropy 25(6):903. \url{https://doi.org/10.3390/e25060903}
\bibitem{sipser2012} Sipser M (2012) Introduction to the Theory of Computation, 3rd edn. Cengage Learning, Boston.
\bibitem{wolfram2002} Wolfram S (2002) A New Kind of Science. Wolfram Media, Champaign.
\bibitem{wolfram1985} Wolfram S (1985) Undecidability and intractability in theoretical physics. Phys Rev Lett 54(8):735-738. \url{https://doi.org/10.1103/PhysRevLett.54.735}
\bibitem{zenil2021} Zenil H, Kiani NA, Tegnér J (2021) The Thermodynamics of Network Coding, and an Algorithmic Refinement of the Principle of Maximum Entropy. Entropy 23(10):1330. \url{https://doi.org/10.3390/e23101330}
\bibitem{bialek2001} Bialek W, Nemenman I, Tishby N (2001) Predictability, complexity, and learning. Neural Comput 13(11):2409-2463. \url{https://doi.org/10.1162/089976601753195969}
\bibitem{beer2014} Beer RD (2014) Dynamical systems and embedded cognition. In: The Cambridge Handbook of Artificial Intelligence. Cambridge University Press, pp 128-148. \url{https://doi.org/10.1017/CBO9781139046855.009}
\bibitem{bertschinger2008} Bertschinger N, Olbrich E, Ay N, Jost J (2008) Autonomy: An information theoretic perspective. Biosystems 91(2):331-345. \url{https://doi.org/10.1016/j.biosystems.2007.05.018}
\bibitem{cook2004} Cook M (2004) Universality in elementary cellular automata. Complex Syst 15(1):1-40.
\bibitem{rice1953} Rice HG (1953) Classes of recursively enumerable sets and their decision problems. Trans Am Math Soc 74(2):358-366. \url{https://doi.org/10.2307/1990888}
\bibitem{li2008} Li M, Vitányi P (2008) An Introduction to Kolmogorov Complexity and Its Applications, 3rd edn. Springer, New York. \url{https://doi.org/10.1007/978-0-387-49820-1}
\bibitem{bennett1988} Bennett CH (1988) Logical depth and physical complexity. In: The Universal Turing Machine: A Half-Century Survey. Oxford University Press, pp 227-257.
\bibitem{moreno2015b} Moreno A, Mossio M (2015) Biological Autonomy: A Philosophical and Theoretical Enquiry. Springer, Dordrecht.
\bibitem{klyubin2005} Klyubin AS, Polani D, Nehaniv CL (2005) Empowerment: A universal agent-centric measure of control. In: Proceedings of the 2005 IEEE Congress on Evolutionary Computation, vol 1. IEEE, pp 128-135. \url{https://doi.org/10.1109/CEC.2005.1554676}
\bibitem{wang2024} Wang Z, Smith JE (2024) Emergence and causality in complex systems: A survey of causal emergence and related quantitative studies. Entropy 26(2):108. \url{https://doi.org/10.3390/e26020108}
\bibitem{kauffman1993} Kauffman SA (1993) The Origins of Order: Self-Organization and Selection in Evolution. Oxford University Press, New York.
\bibitem{hanson2009} Hanson JE (2009) Cellular Automata, Emergent Phenomena in. In: Encyclopedia of Complexity and Systems Science. Springer, pp 998-1009. \url{https://doi.org/10.1007/978-0-387-30440-3_64}
\bibitem{bednarz2013} Bednarz T, Boschetti F, Finnigan J, Grigg N (2013) More complex complexity: Exploring the nature of computational irreducibility across physical, biological, and human social systems. In: Irreducibility and Computational Equivalence. Springer, pp 79-85. \url{https://doi.org/10.1007/978-3-642-35482-3_7}
\bibitem{tononi2003} Tononi G, Sporns O (2003) Measuring information integration. BMC Neurosci 4(1):31. \url{https://doi.org/10.1186/1471-2202-4-31}
\bibitem{dennett2003} Dennett DC (2003) Freedom Evolves. Viking Press, New York.
\bibitem{lloyd2012} Lloyd S (2012) A Turing test for free will. Philos Trans R Soc A Math Phys Eng Sci 370(1971):3597-3610. \url{https://doi.org/10.1098/rsta.2011.0331}
\bibitem{johansen2024} Johansen M, Krueger D, Antonoglou I, Hoffman M (2024) Mechanistic Interpretability for AI Safety Review. arXiv preprint arXiv:2404.14082.
\bibitem{russell2021} Russell S (2021) Human Compatible: Artificial Intelligence and the Problem of Control, Updated edn. Penguin Press, New York.
\bibitem{israeli2012} Israeli N, Goldenfeld N (2012) Empirical encounters with computational irreducibility and unpredictability. Minds Mach 22(4):329-336. \url{https://doi.org/10.1007/s11023-011-9262-y}
\bibitem{gellmann2003} Gell-Mann M, Lloyd S (2003) Effective complexity. In: Nonextensive Entropy: Interdisciplinary Applications. Oxford University Press, pp 387-398. \url{https://doi.org/10.1093/oso/9780195159769.003.0022}
\bibitem{hoel2017} Hoel EP (2017) When the map is better than the territory. Entropy 19(5):188. \url{https://doi.org/10.3390/e19050188}
\bibitem{arthur2023} Arthur R, Chen J, Liang P, Russell S (2023) Autonomous Al systems in conflict: Emergent behavior and its impact on predictability and reliability. J Strateg Stud. \url{https://doi.org/10.1080/15027570.2023.2213985}
\bibitem{ay2014} Ay N, Zahedi K (2014) On the causal structure of the sensorimotor loop. In: Guided Self-Organization: Inception. Springer, pp 261-294. \url{https://doi.org/10.1007/978-3-642-53734-9_9}
\bibitem{wang1961} Wang H (1961) Proving theorems by pattern recognition II. Bell Syst Tech J 40(1):1-41. \url{https://doi.org/10.1002/j.1538-7305.1961.tb03975.x}
\bibitem{wegner1997} Wegner P (1997) Why interaction is more powerful than algorithms. Commun ACM 40(5):80-91. \url{https://doi.org/10.1145/253769.253801}
\bibitem{kauffman2006} Kauffman S, Clayton P (2006) On emergence, agency, and organization. Science 312(5771):200-201. \url{https://doi.org/10.1126/science.1122784}
\bibitem{barbieri2008} Barbieri, M. (2008). Biosemiotics: a new understanding of life. Naturwissenschaften, 95(7), 577-599. \url{https://doi.org/10.1007/s00114-008-0382-7}
\bibitem{barbieri2003} Barbieri, M. (2003). The Organic Codes: An Introduction to Semantic Biology. Cambridge University Press. \url{https://doi.org/10.1017/CBO9780511498679}
\bibitem{barbieri2015} Barbieri, M. (2015). Code Biology: A New Science of Life. Springer. \url{https://doi.org/10.1007/978-3-319-15891-9}
\bibitem{barbieri2016} Barbieri, M. (2016). From biosemiotics to code biology. Biosemiotics, 9(2), 169-184. \url{https://doi.org/10.1007/s12192-016-0683-1}
\bibitem{soloveichik2008} Soloveichik, D., Seelig, G., \& Winfree, E. (2008). DNA as a universal substrate for chemical kinetics. Proceedings of the National Academy of Sciences, 105(39), 14798-14803. \url{https://doi.org/10.1073/pnas.0807661105}
\bibitem{beckage2013} Beckage, B., Gross, L. J., Laclair, R. J., \& Smith, T. (2013). More complex complexity: Exploring the nature of computational irreducibility across physical, biological, and human social systems. In Irreducibility and Computational Equivalence (pp. 79-85). Springer.
\bibitem{gisin2021} Del Santo, F., \& Gisin, N. (2021). The physical implications of the no-go theorems on determinism. arXiv preprint arXiv:2104.09874.
\bibitem{gisin2020} Gisin, N. (2020). Mathematical languages shape our understanding of the physical world. Nature Physics, 16(10), 1018-1020.
\bibitem{jaeger2021} Jaeger, J. (2021). Algorithmic mimicry. Biology \& Philosophy, 36(5), 1-28.
\bibitem{onyeukaziri2022} Onyeukaziri, N. J. (2022). Artificial intelligence and moral responsibility: A critique. Phronimon, 23, 1-17.
\end{thebibliography}
\end{document}